\documentclass[a4paper,UKenglish,cleveref, autoref, thm-restate]{lipics-v2021}

\hideLIPIcs  

\usepackage{times}
\usepackage{soul}
\usepackage{url}
\usepackage{amsmath}
\usepackage{amssymb}
\usepackage{amsthm}
\usepackage{mathtools}
\usepackage{booktabs}
\usepackage{cleveref}
\usepackage{bbm}
\usepackage{multirow}
\usepackage{stmaryrd}
\usepackage{numprint}
\usepackage{enumitem}
\usepackage{caption}
\usepackage{subcaption}
\usepackage{algorithm}
\usepackage[noend]{algpseudocode}

\newcommand{\hide}[1]{}

\DeclareMathOperator*{\argmax}{arg\,max}

\newtheorem{prop}{Proposition}

\title{Guided Bottom-Up Interactive Constraint Acquisition}

\author{Dimosthenis C. Tsouros}{KU Leuven, Belgium}{dimos.tsouros@kuleuven.be}{0000-0002-3040-0959}{}
\author{Senne Berden}{KU Leuven, Belgium}{senne.berden@kuleuven.be}{0000-0002-6473-5757}{}
\author{Tias Guns}{KU Leuven, Belgium}{tias.guns@kuleuven.be}{0000-0002-2156-2155}{}

\authorrunning{D. Tsouros, S. Berden and T. Guns} 

\Copyright{Dimosthenis Tsouros, Senne Berden and Tias Guns} 



\category{} 

\relatedversion{} 

\nolinenumbers 

\EventEditors{Roland H. C. Yap}
\EventNoEds{1}
\EventLongTitle{29th International Conference on Principles and Practice of Constraint Programming (CP 2023)}
\EventShortTitle{CP 2023}
\EventAcronym{CP}
\EventYear{2023}
\EventDate{August 27--31, 2023}
\EventLocation{Toronto, Canada}
\EventLogo{}
\SeriesVolume{280}
\ArticleNo{5}

\begin{document}


\maketitle

\begin{abstract}
Constraint Acquisition (CA) systems can be used to assist in the modeling of constraint satisfaction problems. 
In (inter)active CA, the system is given a set of candidate constraints and posts queries to the user with the goal of finding the right constraints among the candidates.
Current interactive CA algorithms suffer from at least two major bottlenecks. First, in order to converge, they require a large number of queries to be asked to the user. 
Second, they cannot handle large sets of candidate constraints, since these lead to large waiting times for the user. For this reason, the user must have fairly precise knowledge about what constraints the system should consider.
In this paper, we alleviate these bottlenecks by presenting two novel methods that improve the efficiency of CA.
First, we introduce a bottom-up approach named \textsc{GrowAcq} that reduces the maximum waiting time for the user and allows the system to handle much larger sets of candidate constraints. It also reduces the total number of queries for problems in which the target constraint network is not sparse.
Second, we propose a probability-based method to guide query generation
and show that it can significantly reduce the number of queries required to converge.
We also propose a new technique that allows the use of openly accessible CP solvers in query generation, removing the dependency of existing methods on less well-maintained custom solvers that are not publicly available.
Experimental results show that our proposed methods outperform state-of-the-art CA methods, reducing the number of queries by up to $60\%$. Our methods work well even in cases where the set of candidate constraints is 50 times larger than the ones commonly used in the literature.
\end{abstract}

\section{Introduction and related work}
Constraint programming (CP) is considered one of the main paradigms for solving combinatorial problems, with many successful applications in a variety of domains. However, there are still challenges to be faced in order for CP technology to become even more widely used. One of the most important challenges is to ease the modeling process. The current assumption in CP is that the user first models the problem and that a solver is then used to solve it. However, modeling is a non-trivial task. Expressing a combinatorial problem as a set of constraints over decision variables is not straightforward and requires substantial expertise~\cite{freuder1999modeling}. As a result, modeling is considered a major bottleneck for the widespread adoption of CP~\cite{freuder1999modeling,freuder2018progress,freuder2014grand}.

This obstacle has led to research into a very different approach to modeling: that of \textit{learning} the constraint problem from data, as opposed to manually constructing it. This is the focus of the research area of \textit{constraint acquisition (CA)}, in which CP meets machine learning. In CA, the model of a constraint problem is acquired (i.e., learned) (semi-)automatically from a set of examples of solutions, and possibly non-solutions. CA methods can be categorized as \textit{active} or \textit{passive} on the basis of whether a user provides feedback during learning or not.

In {\em passive acquisition}, a dataset of examples of solutions and non-solutions is provided by the user upfront. Based on these examples, the system learns a set of constraints modeling the problem~\cite{beldiceanu2012model,berden2022learning,bessiere2004leveraging,bessiere2005sat,bessiere2017constraint,kumar2019acquiring,kumar2022learning,lallouet2010learning,lombardi2017empirical}. 
Approaches vary in the types of constraints they are able to learn and the methodologies they employ: \textsc{Conacq.1} is a version space algorithm for learning fixed-arity constraints \cite{bessiere2004leveraging,bessiere2005sat,bessiere2017constraint}, ModelSeeker learns global constraints that are taken from a predefined constraint catalog \cite{beldiceanu2012model}, and COUNT-CP is a generate-and-aggregate approach that can learn expressive first-order constraints \cite{kumar2022learning}. 
None of these approaches are robust to errors in the labeled data. To this end, \textsc{SeqAcq} and \textsc{BayesAcq} were introduced, being robust to noise in the training set. In \textsc{SeqAcq}, a statistical approach based on sequential analysis is used \cite{prestwich2020robust}, while in \textsc{BayesAcq}, a naive Bayes classifier is trained, from which a constraint network is then derived \cite{prestwich2021classifier}.

In contrast to passive learning, {\em active} or {\em interactive acquisition} systems learn the constraints through interaction with the user, by asking queries. The main type of query used is the \textit{membership query}, which asks the user to classify a given example (i.e., an assignment to the variables of the problem) as a solution or a non-solution. An early work in active CA is the Matchmaker agent~\cite{freuder1998suggestion}, where users, when they answer a membership query negatively, also have to provide a violated constraint. In order to lower the expertise level required from the user, Bessiere et al. later proposed \textsc{Conacq.2}~\cite{bessiere2007query,bessiere2017constraint} -- an active version of \textsc{Conacq.1} that uses membership queries and does not require the user to provide any violated constraints. In~\cite{shchekotykhin2009argumentation}, \textsc{Conacq.2} was in turn extended to also accept arguments regarding \textit{why} examples should be rejected or accepted.

As the number of membership queries needed can be exponentially large for these methods \cite{bessiere2017constraint}, a new family of interactive algorithms was proposed that use \textit{partial queries} instead~\cite{arcangioli2016multiple,bessiere2013constraint,lazaar2021parallel,tsouros2020efficient,tsouros2021learning,tsouros2019structure,tsouros2020omissions,mquacq}. A partial query asks the user to classify a partial assignment to the variables. Using partial queries, CA systems are able to converge faster. \textsc{QuAcq} was the first system to use partial queries~\cite{bessiere2023learning,bessiere2013constraint}, and was later extended into \textsc{MultiAcq}~\cite{arcangioli2016multiple}. \textsc{MQuAcq} was later introduced to reduce the number of queries needed per learned constraint \cite{tsouros2020efficient,mquacq}, and \mbox{\textsc{MQuAcq-2}} further improved the performance by exploiting the structure of the constraints already learned \cite{tsouros2019structure}. 

Despite these advancements in active CA, there are still significant obstacles for the technology to become usable in practice. 
One of the main limitations is that it typically still requires asking a large number of queries to the user in order to find all constraints.
In addition, existing systems cannot handle large sets of candidate constraints in reasonable run times, and thus require significant expertise from the user in limiting the constraints the system should consider (and thus the size of the candidate set) upfront.
Finally, query generation -- a highly important part of the CA process -- currently requires the use of customized solvers that are not publicly available and are not as well-maintained as conventional solvers. Without the use of such customized solvers, current active CA algorithms can lead to very high query generation times or are sometimes unable to converge to the correct set of constraints when time limits are imposed~\cite{addi2018time,tsouros2020efficient}.

We focus on the above limitations, and contribute the following improvements:

\begin{itemize}

    \item We present a novel query generation method named \textsc{PQ-Gen} that allows conventional constraint solvers to be used by CA algorithms while also ensuring convergence, removing the dependency on customized solvers.
    
    \item We propose a bottom-up learning approach named \textsc{GrowAcq} that uses any other CA algorithm to learn the constraints of an increasingly large problem. It starts learning with only a subset of variables and an associated subset of candidate constraints, and incrementally grows this set of variables and constraints. This allows it to handle significantly larger sets of candidate constraints and reduces the maximum waiting time for the user. 

    \item Finally, we introduce a better way to \textit{guide} the query generation process, with the goal of generating queries that learn the set of constraints faster.
    We propose an objective function for query generation that uses \textit{probabilistic estimates} of whether constraints are likely to hold or not. 
    We demonstrate the potential of this method by using a simple counting-based approach as probabilistic estimator. 
 
\end{itemize}

The rest of the paper is structured as follows. Some background on CA is given in Section~\ref{sec:back}.  \Cref{sec:qgen,sec:growacq,sec:guided} present our proposed methods. An experimental evaluation is given in Section~\ref{sec:exp}. Finally, Section~\ref{sec:concl} concludes the paper.

\section{Background}
\label{sec:back}

We now introduce some basic notions regarding constraint satisfaction problems and interactive constraint acquisition.

\subsection{Constraint satisfaction problems}

\label{def:csp}

A \textit{constraint satisfaction problem} (\textit{CSP}) is a triple $P = (X, D, C)$, consisting of:

\begin{itemize}

\item a set of $n$ variables $X = \{x_1, x_2, ..., x_n\}$, representing the entities of the problem,

\item a set of $n$ domains $D = \{D_1, D_2, ..., D_n\}$, where $D_i \subset \mathbb{Z}$ is the finite set of values for $x_i$,

\item a constraint set (also called constraint network) $C = \{c_1, c_2, ..., c_t\}$.

\end{itemize}

A \textit{constraint} $c$ is a pair ($rel(c)$, $var(c)$), where $var(c)$ $\subseteq X$ is the \textit{scope} of the constraint and $rel(c)$ is a relation over the domains of the variables in $var(c)$ that specifies (implicitly or explicitly) what assignments are allowed. $|var(c)|$ is called the \textit{arity} of the constraint. 
The constraint set $C[Y]$, where $Y \subseteq X$, denotes the set of constraints from $C$ whose scope is a subset of $Y$. The set of solutions of a constraint set $C$ is denoted by $sol(C)$. 
A \textit{redundant} or \textit{implied} constraint $c \in C$ is a 
constraint in $C$ such that $sol(C) = sol(C\setminus \{c\})$. 

A (partial) \textit{assignment} $e_Y$ is an assignment over a set of variables $Y \subseteq X$. $e_Y$ is \textit{rejected} by a constraint $c$ iff $var(c)$ $\subseteq Y$ and the projection $e_{var(c)}$ of $e_Y$ on the variables in the scope $var(c)$, is not in $rel(c)$, that is, is not allowed by the constraint.
$\kappa_C(e_Y)$ represents the subset of constraints from $C[Y]$ that reject $e_Y$, i.e., $\kappa_C(e_Y) = \{c \,|\, c \in C[Y] \land e_{var(c)} \notin rel(c)$ \}.

A complete assignment $e$ that is accepted by all the constraints in $C$ is a \textit{solution} to $C$, i.e., $e \in sol(C)$.
A partial assignment
$e_Y$ is called a \textit{partial solution} to $C$ iff it is accepted by all the constraints in $C[Y]$. Note that a partial solution to $C$ may not be extendable to a complete one, due to constraints not in $C[Y]$.

\subsection{Active constraint acquisition with partial membership queries}
In CA, the pair $(X, D)$ is called the \textit{vocabulary} of the problem at hand and is common knowledge shared by the user and the system. Besides the vocabulary, the learner is also given a \textit{language} $\Gamma$ consisting of {\em fixed-arity} constraint relations. Using the vocabulary $(X, D)$ and the constraint language $\Gamma$, the system generates the \textit{constraint bias} $B$, which is the set of all possible candidate constraints for the problem.

Let $C_T$, the target constraint network, be an unknown set of constraints such that for every assignment $e$ over $X$ it holds that $e \in sol(C_T)$ iff $e$ is a solution to the problem the user has in mind.
The goal of CA is to learn a constraint set $C_L \subseteq B$ that is equivalent to $C_T$. Like other works, we assume that the bias $B$ can represent $C_T$, i.e., there exists a $C \subseteq B$ s.t. $ sol(C) = sol(C_T)$. 

In active CA, the system interacts with the user while learning the constraints.
A \textit{membership query}~\cite{angluin1988queries} in this setting is a question $ASK(e_X)$, asking the user whether a complete assignment $e_X$ is a solution to the problem that the user has in mind. 
A \textit{partial query} $ASK(e_Y)$, with $Y \subset X$, 
asks the user to determine if $e_Y$, which is an assignment in $D^Y$, 
is a partial solution with respect to $C_T[Y]$. 
We use the notation $c \in C_T$  iff $\,\forall \, e \in D^Y$ with $var(c) \subseteq Y \subseteq X$, $ASK(e_Y) = True \implies e_{var(c)} \in sol(c)$.

While in passive acquisition there are methods that can handle noisy answers~\cite{prestwich2020robust,prestwich2021classifier}, this is not the case for active acquisition. For this reason, in this work, we follow the assumption that the user answers all queries correctly.

A query $ASK(e_Y)$ is called \textit{irredundant} iff the answer is not implied by any information already available to the system. That is, the query is irredundant iff $e_Y$ is rejected by at least one constraint from the bias $B$ and is not rejected by the network $C_L$ learned thus far.
The first condition captures that $\kappa_B(e_Y)$ cannot be empty, since if $\kappa_B(e_Y)$ would be empty, the answer to the query $ASK(e_Y)$ would have to be `yes', based on the assumption that $C_T$ is representable by the constraints in $B$. The second condition captures that $e_Y$ should not be rejected by any constraint in the learned network $C_L$, 
since otherwise the user would certainly answer `no' to the query.


\begin{algorithm}[t]
\caption{Constraint Acquisition through partial queries}\label{alg:general}
\begin{algorithmic}[1]

\Require $X$, $D$, $B$, $C_{in}$ ($X$: the set of variables, $D$: the set of domains, $B$: the bias, $C_{in}$: an optional set of known constraints)
\Ensure $C_L$ : a learned constraint network

\State $C_L \leftarrow C_{in}$

\While {True}

	\State Generate an $e$ accepted by $C_L$ and rejected by $B$

	\If{ $e$ = nil } \Return $C_L$  \Comment{Stopping condition}
	\EndIf
    
    \If{$ASK(e)$ = Yes}     \Comment{Ask (partial) membership query $e$}
        \State Remove the constraints rejecting $e$, namely $\kappa_{B}(e)$, from $B$
    \Else

        \State Find one (or more) minimal scopes $S$ in $e$ for which $|\kappa_{B}(e_S)|\geq1$ and $ASK(e_S) =$ No
                
        \State Find all $\{c \in C_T \mid var(c) = S\}$ through partial queries; 
        add to $C_L$, remove from $B$

        
    \EndIf

\EndWhile

\end{algorithmic}
\end{algorithm}


\Cref{alg:general} presents the generic process followed by active CA methods with partial queries. The learned set $C_L$ is first initialized either to the empty set or to a set of constraints given by the user that is known to be part of $C_T$ (i.e., $C_{in} \subset C_T$) (line 1). Then the main loop of the acquisition process begins, where, in every iteration, the system first generates an irredundant query (line 3) and posts it to the user (line 5). If the query is answered positively, then the candidate constraints from $B$ that violate it are removed (line 6). Otherwise, the system has to find one or more constraints from $C_T$ that violate the query. This is done in two steps. First, queries are asked to find the scope of a constraint in $\kappa_{C_T}(e)$ (line 8). Then, queries are asked to find all constraints $c \in C_T$  with that scope (line 9). 

The acquisition process has \textit{converged} on the learned network $C_L \subseteq B$ iff $C_L$ agrees with the set of all labeled examples $E$, and for every other network $C \subseteq B$ that agrees with $E$, it holds that $sol(C) = sol(C_L)$. This is proved if no query could be generated at line 3, as in this case, all remaining constraints in $B$ (if any) are redundant.
If the first condition
is true but the second condition 
has not been proved when the acquisition process finishes,
\textit{premature convergence} has occurred. This can happen when the query generation at line 3 returns $e=nil$, but without having proved that an irredundant query does not exist (e.g., because of a time limit). 

Existing algorithms like \textsc{QuAcq}~\cite{bessiere2023learning,bessiere2013constraint}, \textsc{MQuAcq}~\cite{tsouros2020efficient,mquacq} and \textsc{MQuAcq-2}~\cite{tsouros2019structure} follow this template, but differ mainly in how they implement lines 3, 8 and 9, and hence how many constraints they are able to learn in each iteration. Examples of functions used to locate the scope of a constraint (line 8) are \textit{FindScope}~\cite{bessiere2023learning,bessiere2013constraint} or the more efficient \textit{FindScope-2}~\cite{tsouros2020efficient}. To learn the constraints in the scope found (line 9), the \textit{FindC} function is typically used~\cite{bessiere2023learning,bessiere2013constraint}.

\section{Using conventional solvers for query generation}
\label{sec:qgen}
Query generation (line 3 of~\Cref{alg:general}) 
is one of the most important parts of the CA process. 
It aims to find an \textit{irredudant} membership query (i.e., a (partial) assignment that does not violate $C_L$ but violates at least one $c \in B$) that will be asked to the user. Thus, it can be formalized as follows:

\begin{center}
$\text{find } e_Y \text{ s.t. } e_Y \in sol( C_L[Y] \wedge \bigvee_{c_i \in B[Y]} \neg c_i ),$
\end{center}

which can be formulated as a CSP with variables $Y$ and constraints $C_L[Y] \wedge \bigvee_{c_i \in B[Y]} \neg c_i$.


\subsection{Problems when using conventional solvers}
\label{sect:problems_conventional_solvers}
In principle, this CSP could be solved using any conventional CP solver. 
However, this can lead to issues for the following two reasons. 

\textbf{A large bias}
At the start of the acquisition process, the set of candidate constraints $B$ can be very large. This makes the propagation of the constraint $\bigvee_{c_i \in B[Y]} \neg c_i$ time-consuming, and severely slows down the query-generation process.

\textbf{Indirectly implied constraints}
At the end of the acquisition process, only constraints that are implied by $C_L$ remain in $B$, if any. In this case, it will be impossible to generate a query that does not violate $C_L$ and violates at least one constraint from $B$. However, propagation is often unable to prove such implications when they are indirect and involve multiple variables and constraints. For this reason, solvers internally end up enumerating all possible variable assignments satisfying $C_L$ and checking if the constraint $\bigvee_{c_i \in B[Y]} \neg c_i$ can be satisfied. This can be very time-consuming, and a time limit is usually imposed on query generation, leading to \textit{premature convergence}.


In order to limit the large runtimes in a more advanced way than by simply imposing a time bound~$t$, Addi et al. proposed a method using conventional solvers named \textsc{TQ-Gen}~\cite{addi2018time}. It iteratively tries to solve the query generation problem, by gradually reducing the number of variables taken into account by a proportion $\alpha \in \left]0, 1\right[$, until a query can be generated within a small time limit $\tau$. 
This is repeated until either an irredundant query is generated, or a global time bound $t$ is reached, leading to premature convergence. 
However, choosing the right hyperparameters for $t$ and $\tau$ 
is problem-specific~\cite{addi2018time} and requires tuning, and thus more interaction with the user. 



\subsection{Customized solvers}
To avoid premature convergence, a CP solver can be customized to store partial assignments that satisfy every $c \in C_L[Y]$ and violate at least one $c \in B[Y]$ during the search. 
Given an objective function, such as maximizing the number of assigned variables, in every non-failing node of the search tree it will check the above property and, if fulfilled, store the best-scoring partial assignment. 

As these customized solvers are guaranteed to find valid partial solutions, their use will never lead to premature convergence.
In addition, finding a partial query to return is not time-consuming (especially when combined with specialized search heuristics~\cite{tsouros2020efficient}), even when the bias is large.
However, such custom solvers are not publicly available and are typically not based on the latest version of state-of-the-art solvers. This also means that the corresponding active CA methods are heavily tied to those particular customized solvers.


\smallskip

\subsection{Projection-based Query Generation}
We now introduce a method named \textit{Projection-based Query Generation} (\textsc{PQ-Gen}) that makes it possible to use state-of-the-art conventional solvers for query generation, without premature convergence. 
Our proposed method is shown in \Cref{alg:pqgen}.

\begin{algorithm}[t]
\caption{\textsc{PQ-Gen}: Projection-based Query Generation}\label{alg:pqgen}
\begin{algorithmic}[1]

\Require $C_L$, $B$, $l$, $t$ ($B$: the set of candidate constraints (bias), $C_L$: set of known constraints, $l$: size limit, $t$: time limit)
\Ensure $e$: the query generated

\State $timer.start()$
\State $Y \leftarrow \bigcup_{c \in B} var(c)$
\If{ $|B| > l$ }
    \State $e \leftarrow \textrm{solve}(C_L[Y])$

    \If{ $\exists c \in B: e \notin sol(c)$ }
        \State \Return $e$
    \EndIf
\EndIf

\State $e \leftarrow \textrm{solve}(C_L[Y] \wedge \bigvee_{c_i \in B[Y]} \neg c_i )$
\If{$timer.end() < t$}
    \State $e' \leftarrow \textrm{solve}(C_L[Y] \wedge \bigvee_{c_i \in B[Y]} \neg c_i, \mbox{maximize: } obj, \mbox{time limit:} t - timer.end())$
    \If{$e' \neq nil$}
        \State \Return $e'$
    \EndIf

\EndIf

\State \Return $e$
\end{algorithmic}
\end{algorithm}

\textbf{Avoiding indirectly implied constraints}
A key observation we make is that when generating a query on line 3, it might be that $\bigcup_{c \in B} var(c) \subset X$, that is, some variables have no more candidate constraints in $B$. These have become irrelevant, as both lines 6 and 8 are only concerned with $\kappa_B(e)$, which will not include these variables. So, to generate an irredundant query, it is sufficient to consider only the variables in $B$. 
This is not only faster, but also avoids indirectly implied constraints, as these are indirect through variables not used in $B$.

Thus, 
our proposed query generator projects the variables down to $Y \subseteq X$, with 
$Y = \bigcup_{c \in B} var(c)$, thereby simplifying the problem to finding an assignment over $Y \subseteq X$. This will inherently result in a partial assignment when $Y$ is a strict subset of $X$, without requiring a custom solver.
Thus, we first compute the set of variables $Y$ relevant to the query (line 2), and project $C_L$ down to those variables (on lines 4 and 7). The solver then has to prove that there exists a query that satisfies $C_L[Y]$ and violates at least one constraint from $B$. 

\textbf{Dealing with large biases} As mentioned above, having a large bias $B$ can severely slow down the solver during query generation because propagating the $\bigvee_{c_i \in B[Y]} \neg c_i$ constraint takes a long time. However, we observe that when $B$ contains many constraints, the property that a query $e$ violates at least one of these is usually satisfied without needing to enforce this. Hence, we propose not using this constraint when the bias is larger than some threshold (lines 3 to 6 in \Cref{alg:pqgen}). If in a post-hoc check, it turns out that the generated query violates at least one $c \in B$, it is directly returned (line 6). Otherwise, we again generated a query, this time \textit{with} the constraint enforcing that there must exist a constraint in $B$ that is violated (line 7).

\textbf{Optimizing the query}
The above ensures that we will always find \textit{a} valid query. However, much better queries -- according to some objective function -- can often be found. This would take additional time, but is safe because, since a valid query has already been found, the optimization can always safely be interrupted. 
Given a time limit, we can hence call an optimization solver for the remaining time after a first valid query has been found (lines 8-11).

As expressed in Proposition~\ref{correctness_genq}, Algorithm~\ref{alg:pqgen} is correct.
\begin{prop} 
\label{correctness_genq}
Given a bias $B$, with an unknown target network $C_T$ being representable by $B$, and a learned constraint set $C_L$, if $nil$ is returned by~\Cref{alg:pqgen}, then the system has converged on $C_T[X]$.
\end{prop}

\begin{proof}
When $nil$ is returned by~\Cref{alg:pqgen}, it means that $\nexists e \in sol(C_L[Y] \wedge \bigvee_{c_i \in B[Y]} \neg c_i)$, with $Y = \bigcup_{c \in B} var(c)$, i.e., $\nexists e \in sol(C_L[Y] \wedge \bigvee_{c_i \in B[Y]} \neg c_i)$. In order to prove convergence over all of $X$, we must have $\nexists e \in sol(C_L[X] \wedge \bigvee_{c_i \in B[X]} \neg c_i)$. We will now show that when $Y = \bigcup_{c \in B} var(c)$, it means that

\begin{center}
$\nexists e \in sol(C_L[Y] \wedge \bigvee_{c_i \in B[Y]} \neg c_i) \implies \nexists e \in sol(C_L[X] \wedge \bigvee_{c_i \in B[X]} \neg c_i)$
\end{center}

Assume that \Cref{alg:pqgen} returns $nil$, i.e., that no assignment exists in a $Y \subset X$ that is accepted by $C_L[Y]$ and rejected by $B[Y]$. This means that all the constraints in $B[Y]$ are proved to be implied by the constraints in $C_L[Y]$. Thus, the remaining constraints in $B$, that are not proved to be redundant, are the constraints $c \in B \setminus B[Y]$. When we know that $Y = \bigcup_{c \in B} var(c)$ it means that $B[Y] = B$, so $B \setminus B[Y] = \emptyset$. As a result, in this case, all the constraints in $B$ are proved to be implied. Hence, no assignment that is accepted by $C_L$ and rejected by $B$ exists in $X$.
\end{proof}


\section{Bottom-up Constraint Acquisition}
\label{sec:growacq}



We start by observing that all current active CA algorithms always consider either the full set of variables $X$, or a large subset $Y \subseteq X$, in their top-level loop (lines 2-9 in Algorithm~\ref{alg:general}). 
%
This generally leads to complete or almost-complete queries getting generated (line 3 of Algorithm~\ref{alg:general}). 
However, larger queries are generally harder to answer than smaller queries \cite{tsouros2020efficient}. Also, a large initial query leads to many additional queries getting posed in the scope-finding method on line 8. That is because the worst-case complexity of the best scope-finding methods, in terms of the number of queries required, is $\Theta(log(|Y|))$, where $Y \subseteq X$ is the set of variables considered~\cite{tsouros2020efficient}.

Additionally, by directly considering the whole set of variables, the CA algorithm has to represent and operate on the entire set of candidate constraints (i.e., the bias $B$) at once. The bias is used in many parts of the acquisition process. Hence, the memory requirements and the run time of the acquisition process increase significantly as the bias grows, either because the problems contain more variables or because the language $\Gamma$ given to the system includes a larger number of relations. This means that, in practice, state-of-the-art active CA methods are only applicable to problems with not too many variables or problems for which the user already has relatively precise knowledge about what constraints the system should consider (which corresponds to the bias being small).



\begin{algorithm}[t]
\caption{Growing Acquisition}\label{alg:growacq}
\begin{algorithmic}[1]

\Require $\Gamma$, $X$, $D$, $C_{in}$ ($\Gamma$: the language, $X$: the set of variables, $D$: the set of domains, $C_{in}$: an optional set of known constraints)
\Ensure $C_L$ : a constraint network

\State $C_L \leftarrow \emptyset$

\State $Y \leftarrow \emptyset$ 

\While { $|Y| \leq |X|$ }

    \State $x \leftarrow x \in (X \setminus Y$)
    \State $Y \leftarrow Y \cup \{x\}$
    
    \State $B \leftarrow \{ c \mid rel(c) \in \Gamma \land var(c) \subseteq Y \land x \in var(c)\}$

    \State $C_L \leftarrow$ Acq($Y$, $D^Y$, $B$, $C_L \cup C_{in}[Y]$)



\EndWhile

\State \Return $C_L$

\end{algorithmic}
\end{algorithm}


To improve on this, we propose a novel \textit{meta-}algorithm named \textsc{GrowAcq} (Algorithm~\ref{alg:growacq}). 
The key idea is to call a CA algorithm on an increasingly large subset of the variables $Y \subseteq X$,
each time using only a relevant unexplored subset of the bias.
\textsc{GrowAcq} begins with $Y = \emptyset$ (line 2) and gradually incorporates more variables (lines 3-5). Once a new variable $x_i \in X$ has been added to $Y$, the new problem becomes to find the new $C_T[Y]$. However, as $C_T[Y \setminus \{x_i\}]$ was already found in the previous iterations, the set of constraints to seek is actually $C_T[Y] \setminus C_T[Y \setminus \{x_i\}]$. To find $C_T[Y] \setminus C_T[Y \setminus \{x_i\}]$, any existing active CA algorithm can be used. We represent this with the function \textit{Acq} (line 7). 
In every iteration, only a part of the bias $B$ is needed, namely $B[Y] \setminus B[Y \setminus \{x_i\}]$, and as shown in Lemma~\ref{lemma:bias}, the bias constructed at line 6 is equivalent to $B[Y] \setminus B[Y \setminus \{x_i\}]$.

\begin{lemma}
\label{lemma:bias}
Let $Y_i$ be the set of variables $Y$ in iteration $i$ after line 5 of \Cref{alg:growacq} and $B_i = \{ c \mid rel(c) \in \Gamma \land var(c) \subseteq Y_i \land x_i \in var(c)\}$ be the bias $B$ constructed at line 6 in iteration $i$. It holds that $B_i = B[Y_i] \setminus B[Y_{i-1}]$.
\end{lemma}

\begin{proof}
At line 6 of \Cref{alg:growacq}, the bias $B$ is constructed. For each iteration $i$, it is constructed as $B_i = \{ c \mid rel(c) \in \Gamma \land var(c) \subseteq Y_i \land x_i \in var(c)\}$. For a set of variables $Y_i$, the full bias, which includes all candidate constraints, is $B[Y_i] = \{c \mid rel(c) \in \Gamma \land var(c) \subseteq Y_i\}$. For the previous iteration, as $Y_{i-1} = Y_i \setminus \{x_i\}$, we know that $B[Y_{i-1}] = \{c \mid rel(c) \in \Gamma \land var(c) \subseteq Y_i \setminus \{x_i\} \}$. Thus, the additional constraints that are in $B[Y_i]$ and not in $B[Y_{i-1}]$ are the ones with a scope $var(c) \subseteq Y_i$ for which $x_i \in var(c)$: 
\begin{align*}
B[Y_i] \setminus B[Y_{i-1}] &= \{c \mid rel(c) \in \Gamma \land var(c) \subseteq Y_i\} \setminus \{c \mid rel(c) \in \Gamma \land var(c) \subseteq Y_i \setminus \{x_i\} \}\\
&= \{ c \mid rel(c) \in \Gamma \land var(c) \subseteq Y_i \land x_i \in var(c)\} = B_i
\end{align*}
Hence, it holds that $B_i = B[Y_i] \setminus B[Y_{i-1}]$ for $B_i = \{ c \mid rel(c) \in \Gamma \land var(c) \subseteq Y_i \land x_i \in var(c)\}$.
\end{proof}

This bottom-up approach alleviates the problems described above, i.e., starting from large initial queries and having to represent the whole bias from the beginning, in two ways. First, it naturally leads to partial queries of increasing size in the first step of the `inner' CA system (\Cref{alg:general} line 5). This is valuable since smaller queries are generally easier for the user to answer~\cite{tsouros2020efficient}, and also a smaller initial query leads to a lower worst-case number of additional queries to locate scopes.
Second, since the algorithm only stores and uses a small part of the bias at a time (line 6 of \Cref{alg:growacq}), it is able to handle significantly larger biases than the state-of-the-art. Not representing the whole bias in every iteration does not affect the algorithm's correctness, as we state in \Cref{correctness_growacq}.

\begin{prop} 
\label{correctness_growacq}
Given a bias $B$ built from a language $\Gamma$, with bounded arity constraints, and a target network $C_T$ representable by $B$, \textsc{GrowAcq} is correct (i.e., will learn a constraint set $C_L$ that is equivalent to $C_T$), as long as a correct (i.e., sound and complete) CA algorithm is used in line 7.

\end{prop}

\begin{proof} (Sketch)

Let us now prove that if any correct algorithm is used in line 7 of Algorithm~\ref{alg:growacq} -- like \textsc{QuAcq}, \textsc{MQuAcq} or \textsc{MQuAcq-2} -- \textsc{GrowAcq} remains correct. We will subscript sets with the number of the iteration that they occur in to distinguish between the iterations. Even though the full bias $B$ is never constructed and never kept in memory all at once in \textsc{GrowAcq}, we will still refer to it in this proof and denote it with $B$, i.e., $B = \{c \mid rel(c) \in \Gamma \land var(c) \subseteq X\}$. When we instead write $B_i$, we refer to the part of the bias that is constructed and used in iteration $i$ (line 6 of Algorithm~\ref{alg:growacq}), which is $B[Y_i] \setminus B[Y_{i-1}]$ (Lemma~\ref{lemma:bias}).

{\em Soundness}. 
\textsc{GrowAcq}  adds constraints to $C_L$ only at line 7 of \Cref{alg:growacq}. At that line, only constraints returned from the inner interactive CA algorithm are added to $C_L$. Since the assumption is that a sound algorithm is used in the {\em Acq} function, \textsc{GrowAcq} is sound.

{\em Completeness}.
We prove that \textsc{GrowAcq} is complete by proving by induction that, after each iteration $i$, $C_L$ is equivalent to $C_T[Y_i]$, meaning that after the last iteration, $C_L$ is equivalent to $C_T[X]$.
\textsc{GrowAcq} starts with $Y_1 = \emptyset$, so both $C_T[Y_1]$ and $B_1$ are empty. The first iteration where the algorithm has to actually learn any constraints will be the one where $Y$ grows large enough so that $C_T[Y] \neq \emptyset$. Assume that this happens at iteration $k$. In this case, $C_T[Y_k]$ will be representable by $B_k$, because $B_k = B[Y_k] \setminus B[Y_{k-1}]$ and we know that  $C_T[Y_{k-1}] = \emptyset$. Since $C_T[Y_k]$ is representable by $B_k$, it will be successfully learned in line 7, as long as a complete interactive CA algorithm is used.

Assuming now that $C_L = C_T[Y_{n}]$ holds at the end of the $n$-th iteration, let us now prove that $C_L = C_T[Y_{n+1}]$ will hold at the end of the $n\!+\!1$-th iteration. From the assumption that $C_L = C_T[Y_{n}]$, it follows that $(B[Y_n] \setminus C_L) \cap C_T = \emptyset$. As a result, $B_{n+1}$, being equal to $B[Y_{n+1}] \setminus B[Y_n]$ does not exclude any constraint from $C_T[Y_{n+1}]$ that has not already been learned. From this, it follows that $(C_T[Y_{n+1}] \setminus C_L) \subseteq B_{n+1}$, and thus this set of constraints will be learned in line 7 as long as a complete interactive CA algorithm is used. 
Hence, \textsc{GrowAcq} is complete.
\end{proof}

\section{Guided query generation}
\label{sec:guided}

We now turn our attention to the objective function used 
at line 9 of \Cref{alg:pqgen}. Since when \textsc{GrowAcq} is used, the size of $B$ used in every iteration is reduced, query generation is now often fast, leaving sufficient room for using optimization to find a \textit{good} query.


The objective function used in existing query generation systems~\cite{bessiere2023learning,tsouros2020efficient} tries to maximize the number of constraints from $B$ that are violated by the generated query $e$. The motivation is that this can potentially help shrink the bias faster. The objective function is
$$e = \argmax_e \sum_{c \in B} \llbracket e \not\in sol(\{c\}) \rrbracket$$ 
where $ \llbracket \cdot \rrbracket$ is the Iverson bracket which converts \textit{True}/\textit{False} into 1/0.

However, looking only at the number of violated constraints in $B$ does not fully capture what a good query is:
\begin{itemize}
    \item We want queries that lead to a positive answer to violate many constraints from the bias $B$, as these can then all be removed from $B$, shrinking it faster.
    
    \item On the other hand, we want queries that lead to a negative answer to violate a small number of constraints from $B$, as it allows the CA system to find the conflicting constraint faster.
\end{itemize}
 
Based on this, in order to generate good queries regardless of the user's answer, we want query generation to minimize the violation of constraints that are in the unknown target set $C_T$, seeking a query to which the user's answer will be ``yes''. At the same time, we want to maximize the violation of constraints in $B$ that are not in $C_T$, so that positive answers can shrink the bias faster (the first bullet point above). 
Note that we also have the constraint ensuring that at least one constraint from $B$ has to be violated. This means that when $B\setminus C_T = \emptyset$, we want a minimum number of constraints in $C_T$ that we have not already learned to be violated. This leads to negative queries that violate a small number of constraints in $B$ (the second bullet point above).

Assume we have access to an oracle ${\cal O}$ that tells us whether a constraint $c$ belongs to the unknown target set or not: ${\cal O}(c) = (c \in C_T)$. Using this oracle we can formulate an objective function for query generation, using the reasoning above, as follows:


$$
\sum_{c \in B} \llbracket e \not\in sol(\{c\}) \rrbracket \cdot (1 - |\Gamma| \cdot \llbracket {\cal O}(c) \rrbracket),
$$

On the one hand, every time that the oracle returns \textit{False} for a constraint from the bias that is violated by $e$,  the objective function is increased by 1, thereby maximizing the violation of these constraints.
Conversely, for constraints where ${\cal O}$ returns \textit{True}, we aim to minimize the violations, which requires a reduction in the objective value for each such violated constraint.
However, it is possible that violating a set of constraints $C$ (where $\forall c_i \in C\,|\,{\cal O}(c_i) =$ \textit{False}) may imply the violation of a constraint $c_j$ with ${\cal O}(c_j) =$ \textit{True}.
In such cases, if the reduction in the objective value for violating $c_j$ is not large enough, the system will violate both $C$ and $c_j$, maximizing the objective.
To address this issue, we introduce a ``penalty'' of $|\Gamma|$, which is equal to the upper bound of the number of constraints in each scope.
This ensures that the system prioritizes satisfying a constraint with  ${\cal O}(c_j) =$ \textit{True}, over violating other constraints from $B$.

\textbf{Modeling the oracle}
Observe how the current objective of maximizing violations corresponds to using a model of the oracle $M$ that always answers \textit{False}, i.e., that assumes that none of the candidate constraints belong to $C_T$.
On the other hand, if we used an oracle $M$ that always answers \textit{True}, then the query generation would try to violate as few constraints as possible. However, the $\bigvee_{c_i \in B[Y]} \neg c_i$ constraint would still need to be satisfied, in the extreme case leading every query to violate exactly one constraint from $B$. 
Based on this observation, we propose to model the oracle using the following model $M$, which tries to determine for every constraint $c$ whether violating or satisfying $c$ would lead to the least amount of queries later on in the algorithm.

$$
M(c) = \big( \frac{1}{P[c \in C_T]} \leq log(|Y|) \big)
$$

On the one hand, in the extreme case, the constraints for which $M(c)$ answers \textit{True} will be violated one by one in the later queries (once most of the constraints for which $M(c)$ answers \textit{False} have been dealt with). Let $P[c \in C_T]$ be a probabilistic estimate of whether $c$ is part of $C_T$. Then, if the generated queries would violate the constraints with that probability one by one, we would in expectation need $1/P[c \in C_T]$ queries to find a constraint from $C_T$. For example, for a set of constraints that each has a probability of $25\%$, 1 in every 4 queries is expected to lead to a $c \in C_T$ being learned.

On the other hand, for each constraint $c \in C_T$ for which $M(c)$ answers \textit{False}, a scope-finding procedure is needed to locate the violated constraint. The most efficient functions commonly used to do it (i.e., FindScope~\cite{bessiere2023learning} or FindScope-2~\cite{tsouros2020efficient}) have been shown to require $\Theta(log(|Y|))$ queries to find a violated constraint $c \in C_T$ in the worst case, where $Y$ is the number of variables considered in query generation. As a result, we estimate the number of queries needed in this case as $k \cdot log(|Y|)$, with $k$ a constant. We found $k = 1$ to work well in practice.


\textbf{Probability estimation} To compute the probability $P(c \in C_T)$ of a constraint $c \in B$, we use a simple approach, considering only information from the relations $rel(c)$ of the constraints.  More specifically, to compute $P(c \in C_T)$, we count the number of times a constraint with relation $rel(c)$ has been added to $C_L$, and divide it by the total number of times that such a constraint has been removed from $B$.
Much more advanced estimation techniques, including machine learning methods, can be used for more accurate estimation. We leave this for future work.
\section{Experimental evaluation}
\label{sec:exp}

In this section, we empirically answer the following research questions:

\begin{itemize}

    \item [(Q1)] Does using \textsc{PQ-Gen} with conventional solvers avoid premature convergence, and how do CA systems perform when they use it?
    
    \item [(Q2)] Does \textsc{GrowAcq} (using \textsc{MQuAcq-2}) perform better than using \textsc{MQuAcq-2} directly?

    \item [(Q3)] How does our probability-guided query generation objective function perform compared to the one used in current CA systems?

    \item [(Q4)] How does the combination of our methods perform? 

    \item [(Q5)] How do our methods perform on problems with a huge bias $B$?
    
\end{itemize}

\subsection{Benchmarks}
\label{sec:bench}

We used the following benchmarks:

\textbf{Jigsaw Sudoku.}
The Jigsaw Sudoku is a variant of Sudoku in which the $3 \times 3$ boxes are replaced by irregular shapes. 
It consists of 81 variables with domains of size 9. The target network consists of 811 binary $\neq$ constraints, on rows, columns, and shapes. 
The bias $B$ was constructed using the language $\Gamma = \{\geq, \leq, <,>,\neq, = \}$ and contains $19\,440$ binary constraints.

\hide{
\begin{figure}[tb]
	\begin{center}
		\includegraphics[width=100pt]{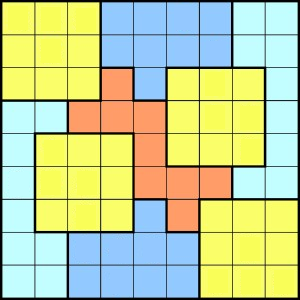}
	\end{center}
	
	\caption{The instance of Jigsaw Sudoku used.} \label{fig:jigsaw}
\end{figure}
}

\textbf{Murder.}
The Murder puzzle problem consists of 20 variables with domains of size 5. The target network contains 4 cliques of 10 $\neq$ constraints and 12 additional binary constraints. The bias was initialized with 760 constraints based on the language $\Gamma = \{\geq, \leq, <,>,\neq, = \}$.

\textbf{Random.}
We used a problem with 100 variables and domains of size 5. We generated a random target network with 495 $\neq$ constraints. The bias was initialized with $19\,800$ constraints, using the language $\Gamma = \{\geq, \leq, <,>,\neq, = \}$.

\textbf{Golomb rulers.}
The problem is to find a ruler where the distance between any two marks is different from that between any other two marks. We built a simplified version of a Golomb ruler with 8 marks, with the target network consisting only of quaternary constraints.\footnote{The ternary constraints derived when $i = k$ or $j = l$ in $|x_i - x_j| \neq |x_k - x_l|$ were excluded, as also done in the literature~\cite{tsouros2020efficient,tsouros2019structure}}
The bias, consisting of 238 binary and quaternary constraints, was created with the language $\Gamma = \{\geq, \leq, <,>,$ $\neq, =$$, |x_i - x_j| \neq |x_k - x_l| \}$.

\textbf{Job-shop scheduling.}
The job-shop scheduling problem involves scheduling a number of jobs, consisting of several tasks, across a number of machines, over a certain time horizon. 
The decision variables are the start and end times of each task.
There is a total order over each job's tasks, expressed by binary precedence constraints. 
There are also constraints capturing the duration of the tasks and that tasks should not overlap on the same machine.
The language $\Gamma = \{\geq, \leq,$ \mbox{$<,$} $>,\neq, =, x_i + c = x_k \}$ was used, with $c$ being a constant from 0 up to the maximal duration of the jobs. We used a problem instance containing 10 jobs, 3 machines (i.e., $|X| = 60$) and a time horizon of 15 steps, leading to a bias containing $14\,160$ constraints.

\subsection{Experimental setup}

Let us now give some details about the experimental settings:

\begin{itemize}

\item
All the experiments were conducted on a system carrying an Intel(R) Core(TM) i9-11900H, 2.50 GHz clock speed, with 16 GB of RAM.

\item 
We measure the total number of queries $\#q$, the average time of the query generation process $\bar{T}_{gen}$ 
(line 3 of \Cref{alg:general}), the average waiting time $\bar{T}$ per query for the user, and the total time needed (to converge) $T_{total}$. All times are presented in seconds.
The difference between $\bar{T}_{gen}$ and $\bar{T}$ is that the latter takes into account also the queries posed on lines 8-9 of \Cref{alg:general}, which are very fast to compute.

\item We evaluate our methods in comparison with the state-of-the-art method \textsc{MQuAcq-2}~\cite{tsouros2019structure}.

\item All methods and benchmarks were implemented in Python~\footnote{Our code is available online in: https://github.com/Dimosts/ActiveConLearn} using the CPMpy constraint programming and modeling library~\cite{guns2019increasing}, except for the experiments using custom solvers.\footnote{For the custom solver based query generators from~\cite{bessiere2023learning,tsouros2020efficient}, we obtained the implementations (in C++) through personal communication with the authors.}

\item The results presented in each benchmark, for each algorithm, are the means of 10 runs.

\end{itemize}  

We now discuss the results of our experimental evaluation, based on the questions we posed at the beginning of the section. 


\subsection{[Q1] Performance of \textsc{PQ-Gen}}
\label{exp:q1}

Both \textsc{PQ-Gen}, our projection-based query generation approach, and \textsc{TQ-Gen}~\cite{addi2018time} (discussed in Section~\ref{sect:problems_conventional_solvers}) involve hyperparameters that affect their performance. Thus, we first performed a hyperparameter sensitivity analysis to assess their performance under different configurations. In tandem with \textsc{TQ-Gen}, we also used the adjust function described in~\cite{addi2018time}. 
We used the JSudoku benchmark for this comparison. For \textsc{TQ-Gen}, we fixed the hyperparameter $\alpha$ to 0.8 as recommended in~\cite{addi2018time}, and used $\tau = \{0.05, 0.1, 0.2, 0.3\}$ and $t = \{0.5, 1, 1.5, 2\}$. For \textsc{PQ-Gen} hyperparameters, we used $l = \{3000, 5000, 7500, 10000\}$ and $t = \{0.5, 1, 1.5, 2\}$. Thus, we examined 16 different configurations for each. A summary of the results are shown in \Cref{ta:tqgen}.\footnote{More details regarding this experiment can be found in \Cref{sec:hyperparameter-exp}.}

\begin{table}[h]
\centering
\caption{A summary of the performance of \textsc{TQ-Gen} and \textsc{PQ-Gen} with different configurations}
\label{ta:tqgen}
{
\begin{tabular}{ lrrrrrr  }

Problem & \multicolumn{1}{c}{$Conv$} & \multicolumn{1}{c}{$\#q$} & \multicolumn{1}{c}{$T_{max}$} & \multicolumn{1}{c}{$T_{total}$} \\
\hline
MQuAcq-2 with \textsc{TQ-Gen~\cite{addi2018time}} & 32\% & 7\,555 & 20.66 & 2\,371.40\\
MQuAcq-2 with \textsc{PQ-Gen} (ours) & 100\% & 6\,551 & 4.42 & 728.25 \\

\end{tabular}
}
\end{table}

Confirming our analysis, with our \textsc{PQ-Gen} there is never a case of premature convergence, no matter what hyperparameters are used. On the other hand, when \textsc{TQ-Gen} is used, the system fails to converge in the majority of cases, and specific hyperparameter values have to be chosen to ensure convergence. 
In addition, our \textsc{PQ-Gen} shows much better performance both in terms of the number of queries needed and, especially, in terms of runtime. 

In more detail, we compared our projection-based query generation (\textsc{PQ-Gen}) with a baseline where we run a conventional CP solver to directly solve the query generation problem, using a one-hour time limit, as well as with query generation methods from the literature, i.e., \textsc{TQ-Gen} and the custom solver based query generators from~\cite{bessiere2023learning,tsouros2019structure}. For \textsc{PQ-Gen} and \textsc{TQ-Gen} we used the best configuration found in the previous experiment. That is, we run \textsc{PQ-Gen} with $l = 5000$ and $t = 1$ and \textsc{TQ-Gen} with $\tau = 0.2$ and $t = 2$. 
We used benchmarks that are similar to the ones used in~\cite{bessiere2023learning,tsouros2019structure}. For consistency, we used the same state-of-the-art query generation objective function across all methods that accept one, i.e., our \textsc{PQ-Gen} and the custom solvers, which tries to maximize the number of violated constraints from $B$. The results are shown in \Cref{ta:qgen}. 

\begin{table}[h]
\centering
\caption{Comparing \textsc{PQ-Gen} with state-of-the-art query generators} 
\label{ta:qgen}

{
\begin{tabular}{ llrrrrrr }

Method & Problem & \multicolumn{1}{c}{$\#q$} & \multicolumn{1}{c}{$\bar{T}_{gen}$} & \multicolumn{1}{c}{$\bar{T}$} & \multicolumn{1}{c}{$T_{max}$} & \multicolumn{1}{c}{$T_{total}$} & \multicolumn{1}{c}{$Convergence$} \\
\hline
\multicolumn{8}{c}{Using conventional solvers}\\
\hline
& JSudoku & 6\,337 & 2.05 & 0.21 & 11.67 & - & 0\%	\\
& Murder & 347 & 0.09 & 0.01 & 0.31 & 4.69 & 100\%	\\
\multirow{-3}{*}{\textsc{MQuAcq-2} $_{\textsc{baseline}}$} & Random &5\,694 & 2.90 & 0.05 & 20.78 & 294.94 & 100\%	\\
\hline
& JSudoku & 7\,153 & 0.26 & 0.13 & 8.46 & 919.34 & 100\% \\
& Murder & 394 & 0.04 & 0.01 & 0.32 & 5.22 & 100\% \\
\multirow{-3}{*}{\textsc{MQuAcq-2} $_{\textsc{TQ-Gen~\cite{addi2018time}}}$} & Random & 5787 & 5.26 & 1.23 & 17.61 & 7100.44
& 100\% \\
\hline
\\ [-0.8em]
& JSudoku & 6\,458 & 0.77 & 0.10 & 3.19 & 666.91 & 100\% \\
& Murder & 370 & 0.66 & 0.03 & 1.10 & 12.28 & 100\% \\
\multirow{-3}{*}{\textsc{MQuAcq-2} $_{\textsc{PQ-Gen (ours)}}$} & Random & 5\,708 & 0.60 & 0.04 & 2.22 & 233.62 & 100\% \\
\hline
\multicolumn{8}{c}{Using custom solvers}\\
\hline
& JSudoku & 5\,321 & 0.99 & 0.06 & 2.04 & 336.42 & 100\% \\
& Murder & 421 & 0.76 & 0.13 & 1.02 & 53.21 & 100\% \\
\multirow{-3}{*}{\textsc{MQuAcq-2} $_{\text{GenerateQuery.cutoff~\cite{bessiere2023learning}}}$} & Random & 5\,349 & 0.94 & 0.04 & 3.24 & 198.65 & 100\% \\
\hline
& JSudoku & 5\,042 & 1.00 & 0.06 & 2.34 & 277.36 & 100\% \\
& Murder & 325 & 0.86 & 0.04 & 1.01 & 12.90 & 100\% \\
\multirow{-3}{*}{\textsc{MQuAcq-2} $_{\text{max$_B$~\cite{tsouros2020efficient}}}$} & Random & 5\,012 & 0.94 & 0.02 & 1.52 & 95.21 & 100\% \\
\hline
\end{tabular}
}
\end{table}

We can observe that convergence was reached in all cases, except for the baseline in which a conventional solver was used directly using a time limit. Our method, \textsc{PQ-Gen}, and the baseline show similar performance in terms of the number of queries needed. while being much better than \textsc{TQ-Gen} in JSudoku and Random, where $B$ is larger, especially when considering time performance.

On the other hand, when custom solvers are used, we can see that the time performance has improved and the number of queries has decreased. This happens because the custom solver can return a partial assignment of any size, trying only to maximize the value of the objective function used, and utilizing heuristics from the literature, while when our \textsc{PQ-Gen} is used, the query generated has to be a solution in a specific (sub)set of variables, which takes more time to compute. 
As a result, we observed that custom solvers often return queries that violate more constraints from $B$, which helps \textsc{MQuAcq-2} shrink the bias faster in terms of the number of queries needed. 


\subsection{[Q2 - Q4] Evaluating \textsc{GrowAcq} and guided query generation}
\label{exp:growacq}\label{exp:guided}

Hereafter, we continue our experiments using \textsc{PQ-Gen}, as the motivation was to investigate techniques that work with any solver. As using PQ-Gen allows us to use conventional solvers, able to run on any given benchmark, in contrast to the custom solvers from~\cite{bessiere2023learning,tsouros2019structure}, where specific constraint relations are implemented, from now on we will use all of the benchmarks mentioned in \Cref{sec:bench}.



\textbf{[Q2] Using \textsc{GrowAcq} within \textsc{MQuAcq-2}}
We now evaluate the performance of \textsc{GrowAcq}, our proposed bottom-up CA approach. To evaluate it, we used \textsc{MQuAcq-2}, as the inner CA algorithm within \textsc{GrowAcq} (line 7 of \Cref{alg:growacq}) and compared this to using \textsc{MQuAcq-2}, directly on the full-sized problem. \Cref{ta:growacq}, top two blocks, presents the results.

We can observe that the usage of \textsc{GrowAcq} results in a reduction of the number of queries in JSudoku, Murder, and Random, while a slight increase can be seen in Golomb. In Job-shop, the increase in the number of queries is somewhat larger ($25\%$). This is the case because the target constraint network in this benchmark is sparse, with most of the iterations of \textsc{GrowAcq} in a $Y \subset X $ not learning any constraint from $C_T$ and only shrinking the bias. So, when the full-sized problem is looked at directly when \textsc{MQuAcq-2} is used, the bias $B$ can shrink with fewer queries. On the other hand, when the target network is not sparse, there is a decrease in the number of queries of up to $19\%$, due to the fact that the system can locate the scopes of the constraints faster, starting from a $Y \subset X$ every time. Based on the above observations, we can see that using \textsc{GrowAcq} leads to learning constraints in a lower amount of queries, but on the other hand, needs more queries to shrink the bias.

Finally, although the total time is almost the same in most problems, and slightly increased in JSudoku and Golomb, the average time per query has not noticeably increased, while the maximum time the user has to wait between two queries has decreased significantly (up to $88\%$ in the Job-shop benchmark), due to the overall reduction in the time needed in query generation in almost all problems (as indicated by the $\bar{T}_{qgen}$ column). As the (maximum) waiting time for the user is of paramount importance for interactive settings, we can see that \textsc{GrowAcq} improves this aspect of time performance of interactive CA systems.

\hide{  
\begin{table}
\centering

\caption{Evaluation of \textsc{GrowAcq} and the proposed approach for guiding query generation}

\label{ta:growacq}

\begin{tabular}{ l|r|rrrrr  }
Problem & \multicolumn{1}{c|}{$\#q$} & \multicolumn{1}{c}{$\bar{T}_{gen}$} & \multicolumn{1}{c}{$\bar{T}$} & \multicolumn{1}{c}{$T_{max}$} & \multicolumn{1}{c}{$T_{total}$} \\
\hline
\multicolumn{6}{c}{MQuAcq-2, \textsc{PQ-Gen}}\\
\hline
JSudoku & 6\,458 & 0.77 & 0.10 & 3.19 & 666.91\\
Murder & 370 & 0.66 & 0.03 & 1.10 & 12.28 \\
Random & 5\,708 & 0.60 & 0.04 & 2.22 & 233.62 \\
Golomb & 233 & 0.96 & 0.22 & 1.22 & 50.89 \\
Job-shop & 590 & 1.03 & 0.10 & 5.36 & 56.23	\\
\hline
\multicolumn{6}{c}{\textsc{GrowAcq} + \textsc{MQuAcq-2}, \textsc{PQ-Gen}}\\
\hline
Sudoku & 5\,863 & 0.15 & 0.12 & 1.98 & 721.15	\\
Murder & 357 & 0.04 & 0.02 & 0.11 & 7.97	\\
Random & 4\,804 & 0.14 & 0.05 & 1.30 & 230.36 \\
Golomb & 270 & 0.80 & 0.29 & 1.30 & 78.47 \\
Job-shop & 786 & 0.13 & 0.06 & 0.66 & 48.18	\\
\hline
\multicolumn{6}{c}{\textsc{GrowAcq} + \textsc{MQuAcq-2} $_{guided}$, \textsc{PQ-Gen}}\\
\hline
JSudoku & 3\,963 & 0.15 & 0.24 & 1.96 & 963.42	\\
Murder & 250 & 0.04 & 0.04 & 0.27 & 9.07	\\
Random & 4\,820 & 0.14 & 0.05 & 1.19 & 229.47	\\
Golomb & 100 & 0.16 & 0.27 & 0.95 & 27.44 \\
Job-shop & 776 & 0.13 & 0.06 & 0.64 & 47.53	\\
\hline
\end{tabular}

\end{table}
}   

\begin{table}
\centering

\caption{Evaluation of \textsc{GrowAcq} and the proposed approach for guiding query generation}

\label{ta:growacq}

\begin{tabular}{ l|r|rrrr  }
Problem & \multicolumn{1}{c|}{$\#q$} & \multicolumn{1}{c}{$\bar{T}_{gen}$} & \multicolumn{1}{c}{$\bar{T}$} & \multicolumn{1}{c}{$T_{max}$} & \multicolumn{1}{c}{$T_{total}$} \\
\hline
\multicolumn{6}{c}{\textsc{MQuAcq-2} }\\
\hline
JSudoku & 6\,458 & 0.77 & 0.10 & 3.19 & 666.91\\
Murder & 370 & 0.66 & 0.03 & 1.10 & 12.28 \\
Random & 5\,708 & 0.60 & 0.04 & 2.22 & 233.62 \\
Golomb & 233 & 0.96 & 0.22 & 1.22 & 50.89 \\
Job-shop & 590 & 1.03 & 0.10 & 5.36 & 56.23	\\
\hline
\multicolumn{6}{c}{\textsc{GrowAcq} + \textsc{MQuAcq-2}}\\
\hline
Sudoku & 5\,863 & 0.15 & 0.12 & 1.98 & 721.15	\\
Murder & 357 & 0.04 & 0.02 & 0.11 & 7.97	\\
Random & 4\,804 & 0.14 & 0.05 & 1.30 & 230.36 \\
Golomb & 270 & 0.80 & 0.29 & 1.30 & 78.47 \\
Job-shop & 786 & 0.13 & 0.06 & 0.66 & 48.18	\\
\hline
\multicolumn{6}{c}{\textsc{GrowAcq} + \textsc{MQuAcq-2} $_{guided}$}\\
\hline
JSudoku & 3\,963 & 0.15 & 0.24 & 1.96 & 963.42	\\
Murder & 250 & 0.04 & 0.04 & 0.27 & 9.07	\\
Random & 4\,820 & 0.14 & 0.05 & 1.19 & 229.47	\\
Golomb & 100 & 0.16 & 0.27 & 0.95 & 27.44 \\
Job-shop & 776 & 0.13 & 0.06 & 0.64 & 47.53	\\
\hline
\end{tabular}

\end{table}


\textbf{[Q3] Guided query generation} In order to evaluate the performance of our proposed objective function for guiding query generation, we compare it with the use of the most popular objective function used in state-of-the-art CA systems, i.e., maximizing violations of constraints from $B$. The objective functions are utilized in line 9 of \Cref{alg:pqgen}. For this comparison, \textsc{GrowAcq} is used, again with \textsc{MQuAcq-2} as the inner acquisition algorithm at line 7 of \Cref{alg:growacq}.
The results using the guided query generation can be seen in~\Cref{ta:growacq}, bottom-two blocks, comparing \textsc{GrowAcq} + 
 \textsc{MQuAcq-2} against \textsc{GrowAcq} + \textsc{MQuAcq-2} $_{guided}$.

We can see that, when using our probability-based guidance for query generation, the number of queries has significantly decreased in JSudoku, Murder, and Golomb, while it has remained nearly the same in Random and Job-Shop. In the latter cases, the number of queries has not decreased because these are under-constrained problems, and thus the probability derived from the constraints' relations was small. This led to maximizing the violations of all constraints in $B$ (i.e., the same behavior as with the existing objective).
On the other hand, in the problems that do not have a sparse constraint network, 
where using the simple counting method to compute the probabilities of the constraints could effectively guide the acquisition system, the decrease observed in the number of queries is substantial (32\% in JSudoku, 30\% in Murder, and 64\% in Golomb). However, as violating constraints one-by-one leads to more queries {\em generated} at line 3 of \Cref{alg:general}, yet fewer queries at lines 8-9, which are very fast to compute, there is a small increase in the total time on JSudoku.


\textbf{[Q4] Combination of our methods} Comparing the combination of our methods (i.e., \textsc{GrowAcq} + \textsc{MQuAcq-2} $_{guided}$) with \textsc{MQuAcq-2} (\Cref{ta:growacq}), we can see that combining our bottom-up approach with guiding the query generation greatly outperforms \textsc{MQuAcq-2} in terms of the number of queries needed to achieve convergence on most of the benchmarks. 
The number of queries has decreased on all benchmarks except Job-shop, where, because of its sparse target network, we need $23\%$ more queries, as \textsc{GrowAcq} increases the number of queries to converge in underconstrained problems,  due to the reasons described in section \ref{exp:growacq}, while guiding the query generation does not improve it, as the probabilities estimated are always low. In the rest of the problems, we observe a total decrease of $16\%$ in Random, $39\%$ in JSudoku, $32\%$ in Murder, and up to $60\%$ in Golomb.

These results demonstrate the effectiveness of the proposed methods in reducing the number of queries needed for CA algorithms, which is crucial in interactive scenarios.

\hide{
\begin{table}
\centering
\caption{Experimental results with guided query generation}
\label{ta:guidedqgen}
{
\begin{tabular}{ lrrrrrr  }

Problem & \multicolumn{1}{c}{$|C_L|$} & \multicolumn{1}{c}{$\#q$} & \multicolumn{1}{c}{$\bar{T}_{gen}$} & \multicolumn{1}{c}{$\bar{T}$} & \multicolumn{1}{c}{$T_{max}$} & \multicolumn{1}{c}{$T_{total}$} \\
\hline
\\ [-0.9em]
\multicolumn{7}{c}{\textsc{GrowAcq} + \textsc{MQuAcq-2}$_{guided}$}\\
\hline
\\ [-0.8em]
JSudoku & 811 & 3,971 & 0.15 & 0.12 & 1.23 & 484.76 \\
Murder & 52 & 264 & 0.04 & 0.03 & 0.22 & 8.22 \\
Random & 495 & 4,813 & 0.14 & 0.03 & 0.99 & 164.88 \\
Golomb & 70 & 100 & 0.16 & 0.27 & 0.95 & 27.44 \\
Job Shop & 50 & 772 & 0.12 & 0.05 & 0.54 & 39.73 \\
\hline

\end{tabular}
}
\end{table}
}

\subsection{[Q5] Dealing with larger biases }

To answer this question, we evaluated \textsc{GrowAcq} and the combination of our methods on larger instances of the Job-shop benchmark, using the same language as before. We used two instances: one with 15 jobs, 11 machines, and 40 steps (denoted as JS-15-11), which resulted in a bias consisting of $542\,850$ constraints, and one with 19 jobs, 12 machines, and again 40 steps (denoted as JS-19-12), resulting in a bias of $1\,037\,400$ constraints. The results are presented in \Cref{ta:largeb}.

On the one hand, \textsc{GrowAcq} needs more queries to converge (like on the smaller Job-Shop instance) because the constraint network of this problem is sparse. 
Yet the total time needed to converge is one order of magnitude lower than in \textsc{MQuAcq-2}, being $24.4$ times faster in the instance with a bias size of $0.5$ million constraints and $25.6$ times faster in the instance with $|B| > 1M$.
In addition, the maximum waiting time has drastically decreased by using \textsc{GrowAcq} (and the combination \textsc{GrowAcq} and guiding query generation), from $5\,499$ seconds to only $3$ (resp. $8$) seconds in JS-15-11 and from more than $20\,371$ seconds to only $7$ (resp. $6$) seconds in JS-19-12. Importantly, the average waiting time is more than $30$ times lower when using \textsc{GrowAcq}. Note that, as in the smaller job-shop instance, guiding does not lead to improvement in terms of the number of queries. 
However, it does not noticeably worsen the time performance of the system. 

Hence, the experiments confirm that the proposed methodology can efficiently handle significantly larger sets of candidate constraints than the state of the art, up to 50 times larger than the ones commonly used in the literature~\cite{bessiere2023learning, bessiere2013constraint,tsouros2019structure,mquacq}. 

\begin{table}[h]
\centering
\caption{Experimental results on instances with a large bias }
\label{ta:largeb}
{
\begin{tabular}{ lrrrrrr  }

Problem & \multicolumn{1}{c}{$|B|$}  & \multicolumn{1}{c}{$\#q$} & \multicolumn{1}{c}{$\bar{T}_{qgen}$} & \multicolumn{1}{c}{$\bar{T}$} & \multicolumn{1}{c}{$T_{max}$} & \multicolumn{1}{c}{$T_{total}$} \\
\hline
\multicolumn{7}{c}{\textsc{MQuAcq-2}}\\
\hline
JS-15-11 & $\approx$ 0.5M & 5\,456 & 66.12 & 6.25 & 5\,499.76 & 34\,085.73	\\
JS-19-12 & $\approx$ 1M & 8\,012 & 80.99 & 9.75 & 20\,371.74 & 78\,124.41	\\
\hline
\multicolumn{7}{c}{\textsc{GrowAcq} + \textsc{MQuAcq-2}}\\
\hline
JS-15-11 & $\approx$ 0.5M & 7\,015 & 0.44 & 0.20 & 2.84 & 1\,422.93	\\
JS-19-12 & $\approx$ 1M & 10\,309 & 0.62 & 0.29 & 6.92 & 2\,984.77	\\
\hline
\multicolumn{7}{c}{\textsc{GrowAcq} + \textsc{MQuAcq-2}$_{guided}$}\\
\hline
JS-15-11 & $\approx$ 0.5M & 7\,062 & 0.44 & 0.20 & 7.88 & 1\,399.63	\\
JS-19-12 & $\approx$ 1M & 10\,219 & 0.64 & 0.30 & 6.19 & 3\,054.68	\\
\hline

\end{tabular}
}
\end{table}

\section{Conclusions}
\label{sec:concl}

Some of the most important limitations of interactive CA methods are the large number of queries needed to converge, as well as the size of the candidate constraint set that they can handle efficiently. In this work, we presented novel methods to alleviate these
issues, improving the efficiency of CA systems. We proposed a bottom-up approach, which allows the system to handle significantly larger biases, reducing the maximum waiting time for the user, and also reducing the total number of queries needed when the target constraint network is not sparse. We also introduced a probabilistic method to guide query generation, further reducing the number of posted queries when our simple counting method could guide the acquisition system to learn constraints more efficiently. In addition, we presented a new query generation technique, named \textsc{PQ-Gen}, that allows the use of conventional CP solvers, removing the dependency of existing methods on customized solvers to converge. Our experimental evaluation showed that our proposed methods outperform state-of-the-art systems in terms of the number of queries in problems with non-sparse constraint networks, reducing this number up to 60\%. In addition, the experiments show that \textsc{GrowAcq} can handle up to 50 times larger biases than the ones commonly used in the literature, allowing CA to tackle increasingly large and complex problems.
The biggest avenue for future work is to further investigate additional ways to reduce the number of queries needed, e.g., by using guidance in all parts of the acquisition process (not just the query generation), and with more advanced probabilistic models. Another important avenue is to consider the setting in which user answers can be noisy 
as has been investigated for passive systems.

\bibliography{paper}

\newpage

\appendix
\section{Hyperparameter evaluation for \textsc{PQ-Gen} and \textsc{TQ-Gen}} \label{sec:hyperparameter-exp}

\begin{figure}[h]

     \centering
     \begin{subfigure}[b]{0.49\textwidth}
        \captionsetup{justification=centering}
         \centering
         \includegraphics[width=\textwidth]{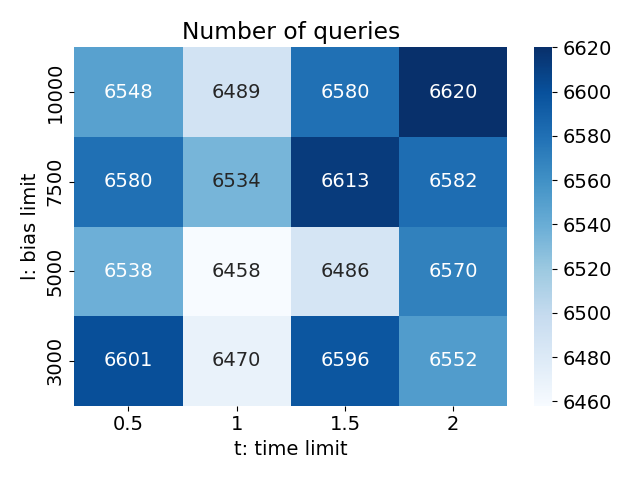}
         \caption{}
     \end{subfigure}
     \hfill
     \begin{subfigure}[b]{0.49\textwidth}
        \captionsetup{justification=centering}
         \centering
         \includegraphics[width=\textwidth]{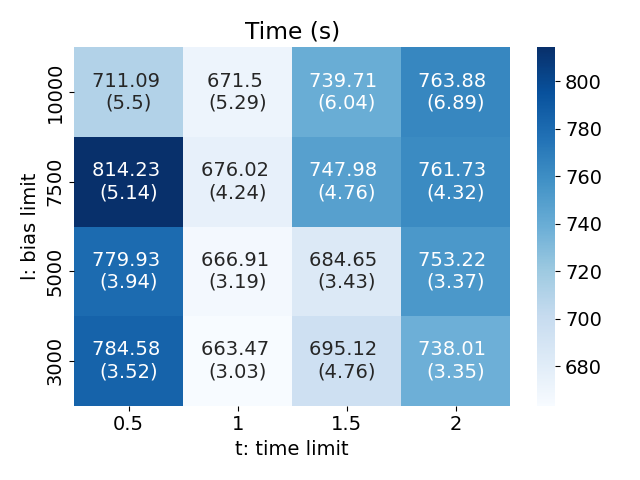}
         \caption{}
     \end{subfigure}
     
     \caption{Performance of \textsc{PQ-Gen} with different hyperparameters values, in terms of: a) the number of queries posted and b) the time (s) needed (in brackets we show the maximum waiting time for the user)}
    \label{fig:pqgen}
\end{figure}

\begin{figure}[h]

     \centering
     \begin{subfigure}[b]{0.49\textwidth}
        \captionsetup{justification=centering}
         \centering
         \includegraphics[width=\textwidth]{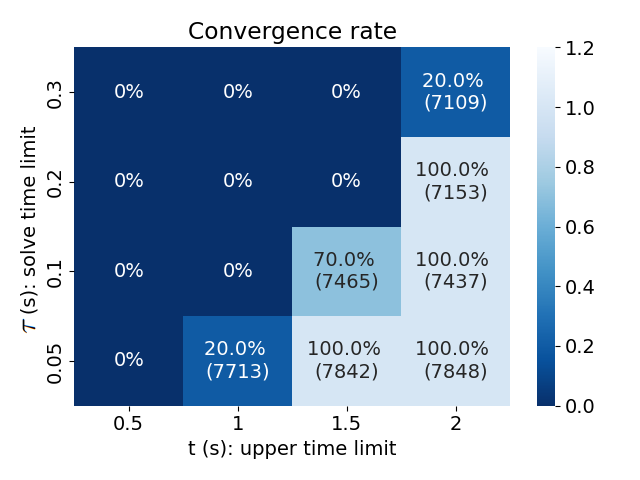}
         \caption{}
     \end{subfigure}
     \hfill
     \begin{subfigure}[b]{0.49\textwidth}
        \captionsetup{justification=centering}
         \centering
         \includegraphics[width=\textwidth]{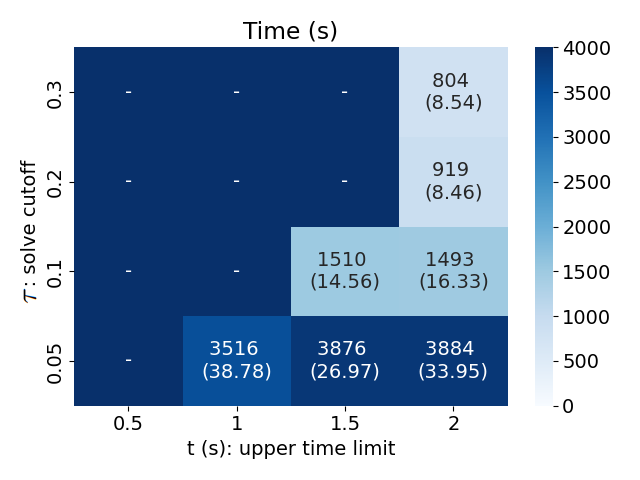}
         \caption{}
     \end{subfigure}
     
     \caption{Performance of \textsc{TQ-Gen} with different parameters, in terms of: a) the convergence rate (in brackets we show the number of queries posted when it converged) and b) the time (s) needed (in brackets we show the maximum waiting time for the user)}
    \label{fig:tqgen}
\end{figure}

Both \textsc{PQ-Gen}, our projection-based query generation approach, and \textsc{TQ-Gen}~\cite{addi2018time} (discussed in Section~\ref{sect:problems_conventional_solvers}) involve hyperparameters that affect their performance. As mentioned in \Cref{exp:q1}, we performed a sensitivity analysis of the performance with respect to the hyperparameter configuration used of \textsc{PQ-Gen} and \textsc{TQ-Gen}~\cite{addi2018time}. In this comparison, both query generation methods were used within the state-of-the-art active CA method \textsc{MQuAcq-2}. We used the JSudoku benchmark for this comparison, as from the benchmarks considered in this paper, this is shown to be the hardest one to reach convergence on (see Table~\ref{ta:qgen}).

In more detail, we varied the hyperparameters of both \textsc{PQ-Gen} and \textsc{TQ-Gen} to assess their performance under different configurations. 
While we fixed the hyperparameter $\alpha$ of ~\cite{addi2018time} to 0.8 as recommended in~\cite{addi2018time}, we had to use different values for the time-related hyperparameters, $\tau$ and $t$, as the previous study used a different CP solver. To be more specific, while we used the CPMpy modeling language, compiling to the OR-Tools CP-SAT solver, the authors of~\cite{addi2018time} use Choco Solver. Although Choco solver is generally inferior to OR-Tools in harder-to-solve problems, OR-tools involves also the pre-solve process in the beginning, which uses most of the (limited) time to generate a query with \textsc{TQ-Gen}, and thus needs larger values for the time limits than when Choco is used as a solver.

In our evaluation we used $\tau = [ 0.05s, 0.1s, 0.2s, 0.3s ]$ and $t = [0.5s, 1s, 1.5s, 2s]$ for \textsc{TQ-Gen}. We also used the adjust function described in~\cite{addi2018time}, as it has been shown to improve its performance.
For \textsc{PQ-Gen} hyperparameters, we used $l = \{3000, 5000, 7500, 10000\}$ and $t = \{0.5, 1, 1.5, 2\}$. Thus, we examined 16 different configurations for each. The results of our experiments are presented in Figures~\ref{fig:pqgen} and \ref{fig:tqgen}, respectively, for \textsc{PQ-Gen} and \textsc{TQ-Gen}.


Focusing on~\Cref{fig:pqgen}, 
we can see that the performance of \textsc{PQ-Gen} is stable across all configurations, both in terms of the number of queries and time performance, having also converged in all cases.  
Let us now shift our focus to~\Cref{fig:tqgen} and the performance of \textsc{TQ-Gen}. 
The first observation is that in the majority of the cases, \textsc{MQuAcq-2} failed to converge when using \textsc{TQ-Gen} as the query generator. 
Only when the time limit was set to 2s, we see at least one run achieving convergence for all values of $\tau$.
In addition, the performance of \textsc{MQuAcq-2} using \textsc{TQ-Gen} is highly sensitive to changes in hyperparameter values, particularly with respect to time.

Overall, comparing the results of \textsc{PQ-Gen} and \textsc{TQ-Gen}, we observe that \textsc{PQ-Gen} exhibits superior performance in terms of convergence rate, fully overcoming the issue of premature convergence. \textsc{PQ-Gen} also requires a lower number of queries to reach convergence and offers improved time performance, resulting in reduced waiting times for the user.

\end{document}